\documentclass[lettersize,journal]{IEEEtran}
\usepackage{amsthm}
\usepackage{amsmath,amssymb,amsfonts}
\usepackage[ruled,linesnumbered]{algorithm2e}
\usepackage{algorithmic}
\usepackage{array}
\usepackage{lmodern}
\usepackage[caption=false,font=normalsize,labelfont=sf,textfont=sf]{subfig}
\usepackage{textcomp}
\usepackage{stfloats}
\usepackage{url}
\usepackage{hyperref}
\usepackage{verbatim}
\usepackage{graphicx}
\usepackage{cite}  
\usepackage{booktabs}
\usepackage{multirow}
\newtheoremstyle{colon}
{0pt} 
{0pt} 
{\normalfont} 
{0em} 
{\bfseries} 
{} 
{0.5em} 
{\thmname{#1}\thmnumber{ #2}\thmnote{ (#3)}} 
\theoremstyle{colon}
\newtheorem{theorem}{Theorem}
\newtheorem{corollary}{Corollary}
\newtheorem{lemma}{Lemma}
\newtheorem{definition}{Definition}
\SetAlgoNlRelativeSize{0}

\begin{document}

\title{An Experimental Approach for Running-Time Estimation of Multi-objective Evolutionary Algorithms in Numerical Optimization}

\author{Han Huang, Tianyu Wang, Chaoda Peng*, Tongli He, Zhifeng Hao}


\maketitle

\begin{abstract}
Multi-objective evolutionary algorithms (MOEAs) have become essential tools for solving multi-objective optimization problems (MOPs), making their running time analysis crucial for assessing algorithmic efficiency and guiding practical applications. While significant theoretical advances have been achieved for combinatorial optimization, existing studies for numerical optimization primarily rely on algorithmic or problem simplifications, limiting their applicability to real-world scenarios. To address this gap, we propose an experimental approach for estimating upper bounds on the running time of MOEAs in numerical optimization without simplification assumptions. Our approach employs an average gain model that characterizes algorithmic progress through the Inverted Generational Distance metric. To handle the stochastic nature of MOEAs, we use statistical methods to estimate the probabilistic distribution of gains. Recognizing that gain distributions in numerical optimization exhibit irregular patterns with varying densities across different regions, we introduce an adaptive sampling method that dynamically adjusts sampling density to ensure accurate surface fitting for running time estimation. We conduct comprehensive experiments on five representative MOEAs (NSGA-II, MOEA/D, AR-MOEA, AGEMOEA-II, and PREA) using the ZDT and DTLZ benchmark suites. The results demonstrate the effectiveness of our approach in estimating upper bounds on the running time without requiring algorithmic or problem simplifications. Additionally, we provide a web-based implementation to facilitate broader adoption of our methodology. This work provides a practical complement to theoretical research on MOEAs in numerical optimization.

\end{abstract}


\section{Introduction}\label{introduction}
\IEEEPARstart{M}{ulti-objective} evolutionary algorithms (MOEAs) play a pivotal role in solving multi-objective optimization problems (MOPs). With the growing complexity of real-world applications, evaluating the computational efficiency of MOEAs has become increasingly critical for practitioners who need to allocate computational resources effectively and select appropriate algorithms for specific problem domains \cite{coello2007evolutionary, zhou2011multiobjective}. Running time analysis provides essential insights for algorithm selection, parameter tuning, and resource planning in practical deployments, particularly in scenarios where computational budgets are limited or optimization campaigns must meet strict time constraints \cite{von2019overview, okoth2022large}. Such analysis enables researchers to evaluate performance differences among different algorithms and provides essential guidance for practical implementations, hence improving the algorithm's performance and adaptability to a wide range of application scenarios \cite{doerr2021theoretical, belaiche2021parallel}.

There are two fundamental concepts for evaluating the running time of evolutionary algorithms: the first hitting time (FHT) and the expected first hitting time (EFHT). FHT represents the number of generations for MOEAs to first reach the target value \cite{he2016average}, while EFHT indicates the average number of generations required, thereby characterizing the average-case computational complexity of MOEAs \cite{yu2008new}. To analyze these quantities, researchers have developed sophisticated theoretical tools that provide rigorous mathematical foundations. Drift analysis \cite{he2004study,he2016average} characterizes the expected progress toward optimization targets by analyzing the expected reduction in distance to the target over generations, while switch analysis \cite{yu2014switch} determines an algorithm's time complexity through comparison with a reference algorithm of known complexity. 

While these theoretical tools provide mathematical foundations, their application to real-world MOEAs presents significant practical challenges. The inherent stochasticity of MOEAs and the complexity of their evolutionary operators, such as crossover, mutation, and selection mechanisms (e.g., selection pressure, diversity maintenance strategies), create substantial analytical complexity. These challenges primarily arise from two interrelated aspects: the stochastic nature of evolutionary processes makes it difficult to derive precise probability distributions of fitness improvements, and the complex interactions among different operators resist straightforward mathematical characterization, making it hard to quantify their combined impact on convergence behavior \cite{huang2019experimental}.

To address these analytical complexities, the research community has adopted systematic simplification strategies that enable rigorous mathematical analysis while maintaining essential algorithmic characteristics. These approaches can be categorized into three methodological progressions based on their abstraction levels. The first category employs simplifications in both algorithmic mechanisms and problem settings, where researchers analyze synthetic MOEAs with simplified operators on carefully constructed test problems with well-defined mathematical properties \cite{laumanns2002running,laumanns2004running,chen2014runtime}. This approach provides fundamental theoretical insights into evolutionary dynamics. The second category maintains realistic problem settings while employing synthetic MOEAs with mathematically tractable structures \cite{neumann2007expected,gao2019runtime,qian2023can}, enabling the study of algorithm behavior under practical problem constraints. The third category analyzes practical algorithms such as NSGA-II \cite{deb2002fast} but incorporates certain operator simplifications, typically utilizing basic genetic operators such as single-point or multi-point crossover on binary strings and bit flip mutation \cite{zheng2022first,li2015primary,huang2019running,huang2020runtime}. While these methodological approaches have advanced our theoretical understanding significantly, the majority of results focus on combinatorial optimization problems, with numerical multi-objective optimization remaining a relatively underexplored area \cite{chen2014runtime}.

Recent work has demonstrated promising directions for analyzing running time without algorithmic or problem simplifications. Notably, Huang et al. \cite{huang2019experimental} developed a fitness-difference drift model called the average gain model for estimating the running time of EAs in numerical optimization problems. However, this methodology remains limited to single-objective optimization scenarios. Extending this approach to multi-objective cases is highly desirable for two primary reasons. First, MOPs are ubiquitous across various domains and predominantly manifest as numerical optimization problems \cite{zhan2022survey}, making running time estimation methods for multi-objective cases of paramount importance. Multi-objective optimization introduces substantial complexity as it requires tracking convergence toward entire Pareto fronts rather than single optimal points, necessitating fundamentally different progress measurement and convergence criteria. Second, the intricate interactions between evolutionary operators in multi-objective settings make rigorous theoretical analysis particularly challenging in continuous domains. Therefore, experimental approaches serve as essential complementary tools, providing reliable running time estimates when theoretical methods become intractable due to these analytical complexities \cite{huang2019experimental}.

We propose an experimental approach for estimating upper bounds on the EFHT of practical MOEAs in numerical optimization problems. We focus on upper bound estimation as it provides essential performance guarantees for practical deployment. Our approach employs an average gain model where the gain is defined as the difference between consecutive IGD \cite{van1998evolutionary} values (see Definition \ref{def:2}) to quantify algorithmic progress. Given the stochasticity of MOEAs and complexity of evolutionary operators, deriving the gain's probabilistic distribution analytically is challenging. To cope with this challenge, we introduce an adaptive sampling method to empirically estimate this distribution, enabling upper bound calculation through established drift analysis principles.
The main contributions of this paper are summarized as follows:
\begin{itemize}
	\item We design an experimental approach using adaptive sampling to estimate upper bounds on EFHT of MOEAs for numerical MOPs without relying on algorithmic or problem simplifications. The framework uses IGD-based gain measurements and adaptive sampling to enhance surface fitting accuracy, offering a complement to theoretical analysis in scenarios where rigorous mathematical approaches face complexity barriers.
	
	\item We validate the proposed approach on five representative MOEAs using ZDT and DTLZ benchmark suites. The experimental results show that the estimated upper bounds maintain consistent relationships with empirical running times across tested algorithms and problems, indicating that the framework can provide useful approximations for algorithm comparison in similar optimization scenarios.
	
	\item We provide a web-based implementation of our framework\footnote{The experimental results in this paper were obtained using the system available at \url{http://www.eatimecomplexity.net/}. By uploading the sample data files of the average gain, the system automatically performs surface fitting, deriving an estimated running time result for the algorithm. The language can be switched to English in the top-right corner of the page.}, enabling researchers and practitioners to easily apply our methodology to their own MOEAs and optimization problems, thereby facilitating broader adoption of upper bound estimation techniques in real-world scenarios.
\end{itemize}
The remainder of this paper is organized as follows. Section \ref{related_work} provides a comprehensive review of running time analysis in MOEAs. Section \ref{avg model} introduces the theoretical foundation of average gain in MOEAs. In Section \ref{Approach}, we present a detailed experimental methodology, including an adaptive sampling method for surface fitting. Section \ref{Experiment} presents the experimental studies and discusses the results in depth. Finally, Section \ref{conclusion} draws a conclusion of this paper.

\section{Related Work}\label{related_work}
Research on running time analysis of MOEAs has made substantial progress over the past decades \cite{laumanns2002running,laumanns2004running,chen2014runtime,neumann2007expected,gao2019runtime,zheng2022first,li2015primary,huang2019running,huang2020runtime,bian2018general,friedrich2010plateaus,giel2006effect,friedrich2011illustration,neumann2012computational,doerr2013lower,doerr2016runtime,qian2013analysis,osuna2017speeding,horoba2009analysis,neumann2010crossover,huang2021runtime,qian2023can,10056396,10266760,10663481,10454586}. However, the majority of existing research focuses on combinatorial optimization problems.
As discussed in Section \ref{introduction}, existing studies can be categorized into three approaches: studies analyzing synthetic MOEAs on synthetic problems, studies using synthetic MOEAs on realistic problems, and studies employing practical algorithms with simplified operators.

The first category achieves theoretical tractability through comprehensive simplifications in both algorithmic mechanisms and problem settings. Laumanns \textit{et al}. \cite{laumanns2002running} established the foundation of population-based MOEA analysis by introducing synthetic algorithms SEMO and FEMO that employ basic operators such as bit flip mutation on the LOTZ problem. This work was extended through general analytical tools including upper-bound techniques based on decision space partitioning and randomized graph search algorithms \cite{laumanns2004running}, enabling analysis of various synthetic algorithms on combinatorial problems such as LOTZ, COCZ, mLOTZ, and mCOCZ. Chen \textit{et al}. \cite{chen2014runtime} further contributed to this category by analyzing $(1+\mu)$ MOEA on a simple continuous two-objective optimization problem (SCTOP), proving that the algorithm can find the first Pareto optimal solution in expected runtime $\mathcal{O}(\sigma^2)$. These foundational studies established essential theoretical frameworks that enable mathematical analysis of evolutionary multi-objective optimization.

The second category employs synthetic algorithms on realistic optimization problems, balancing theoretical tractability with practical relevance. A significant advancement was achieved through GSEMO, an extension of SEMO with more versatile mutation operators that enables analysis on practical problems such as the bi-objective minimum spanning tree problem \cite{neumann2007expected}. Neumann \textit{et al}. proved that GSEMO achieves a 2-approximation of the Pareto front within expected pseudo-polynomial time while remaining mathematically tractable. This approach has been successfully extended to multi-objective spanning tree optimization \cite{gao2019runtime} and multi-objective shortest path problems, where polynomial-time approximation guarantees have been established \cite{horoba2009analysis,neumann2010crossover}. These studies demonstrate the value of maintaining realistic problem settings while using simplified algorithmic frameworks for theoretical analysis.

The third category focuses on practical algorithms while employing operator simplifications to maintain analytical tractability. The first mathematical runtime analysis of NSGA-II was conducted by Zheng \textit{et al}. \cite{zheng2022first}, analyzing its performance on OneMinMax and LOTZ benchmarks using simplified sorting procedures without crossover operations. Subsequent work by Bian and Qian proved that NSGA-II's expected running time for LOTZ is $O(n^3)$ \cite{bian2022better}, while Lu \textit{et al}. demonstrated faster convergence for interactive NSGA-II variants \cite{Lu_Bian_Qian_2024}. In parallel, Li \textit{et al}. \cite{li2015primary} pioneered runtime analysis of the MOEA/D framework using mutation as the sole offspring generation mechanism. This work was later extended by Huang \textit{et al}. to include adaptive operator selection \cite{huang2019running,huang2020runtime}. Recent advances include polynomial acceleration results for archive-based MOEAs \cite{ijcai2024p763} and randomized population update strategies \cite{ijcai2023p612}, as well as approximation ratio analysis for MAP-Elites on NP-hard problems \cite{ijcai2024p773}. These studies represent significant progress toward analyzing practical algorithms, primarily on well-defined benchmark problems.

The majority of existing running time analysis research focuses on combinatorial optimization problems, with very limited theoretical results addressing numerical multi-objective optimization \cite{chen2014runtime}. Moreover, running time analysis of MOEAs in continuous domains typically requires algorithmic or problem simplifications to maintain analytical tractability \cite{huang2019experimental}. While MOEAs have achieved remarkable success in numerical optimization applications, the theoretical understanding of their running time behavior in these domains remains an active area for further development. This gap presents opportunities for complementary approaches that can provide practical insights into algorithm performance without requiring extensive simplifications.

\section{Average IGD Gain Model of MOEAs}\label{avg model}
In this section, we establish the theoretical foundation for our experimental approach to estimating the upper bounds on EFHT of MOEAs. We begin by reviewing the fundamental concepts of multi-objective optimization and the average gain model originally developed for single-objective scenarios. Subsequently, we extend this theoretical framework to multi-objective cases by introducing an IGD-based gain measurement mechanism, which captures the progress of MOEAs toward the entire Pareto front. Finally, we present a theoretical result that enables upper bound estimation based on the average IGD gain, providing the mathematical foundation for our subsequent experimental methodology.

\subsection{Preliminaries}

An MOP can be formally described as
\begin{equation}\label{eq1}
	\begin{aligned}
		&\min \;\; \boldsymbol{F}(\boldsymbol{x}) = \bigl(f_{1}(\boldsymbol{x}), f_{2}(\boldsymbol{x}), \ldots, f_{m}(\boldsymbol{x})\bigr)^T,
	\end{aligned}
\end{equation}
where $m \geq 2$ represents the number of objectives, $\Omega$ denotes the decision space, $\boldsymbol{x} \in \Omega$ is a decision vector, and $\boldsymbol{F}(\boldsymbol{x}): \Omega \rightarrow \mathbb{R}^m$ is the objective vector function mapping from the decision space to the $m$-dimensional objective space.

In multi-objective optimization, the concept of optimality is generalized through Pareto dominance \cite{miettinen1999nonlinear}, which enables comprehensive solution comparison across all objectives simultaneously.

\begin{definition}[Pareto Dominance]
	Given two solutions \(\boldsymbol{x}_a, \boldsymbol{x}_b \in \Omega\), we say that \(\boldsymbol{x}_a\) Pareto dominates \(\boldsymbol{x}_b\), denoted by \(\boldsymbol{x}_a \prec \boldsymbol{x}_b\), if and only if
	$
	\forall i \in \{1, \ldots, m\},\ f_i(\boldsymbol{x}_a) \leq f_i(\boldsymbol{x}_b) \quad $ and $ \quad \exists j \in \{1, \ldots, m\},\ f_j(\boldsymbol{x}_a) < f_j(\boldsymbol{x}_b).
	$
\end{definition}

The Pareto dominance relation establishes a partial ordering in the objective space, enabling the identification of non-dominated solutions.

\begin{definition}[Pareto Optimal Set]
	A decision vector $\boldsymbol{x}^* \in \Omega$ is Pareto optimal if no other solution in the decision space dominates it. Formally, the Pareto optimal set is defined as
	\begin{equation}
		\text{PS} = \bigl\{\boldsymbol{x} \in \Omega \mid \nexists\,\boldsymbol{y} \in \Omega \text{ such that } \boldsymbol{y} \prec \boldsymbol{x}\bigr\}.
	\end{equation}
\end{definition}

\begin{definition}[Pareto Front]
	The Pareto front is the image of the Pareto optimal set in the objective space, defined as
	\begin{equation}
		\text{PF} = \bigl\{\boldsymbol{F}(\boldsymbol{x}) \mid \boldsymbol{x} \in \text{PS}\bigr\}.
	\end{equation}
\end{definition}

\subsection{Average Gain Theory}

Inspired by the concept of pointwise drift \cite{lengler2020drift}, the average gain model serves as a theoretical tool for analyzing the runtime of EAs in continuous solution spaces \cite{huang2014runtime,yushan2016first}. This model was initially proposed by Huang \textit{et al}. \cite{huang2014runtime} for analyzing the (1+1) EA in continuous domains. Zhang \textit{et al}. \cite{yushan2016first} subsequently generalized the model by decoupling it from specific algorithms and objective functions, transforming it into a general analytical framework for studying EA runtime from an abstract perspective. More recently, Huang \textit{et al}. \cite{huang2019experimental} adapted the model for practical applications in single-objective numerical optimization, demonstrating its utility beyond theoretical analysis.

Let $\mathit{{\bar P}_{t}=\left\{p_{1}, p_{2}, \ldots, p_{|{\bar P}_{t}|}\right\}}$ represent the offspring population at generation $\mathit{t}$, with $|{\bar P}_{t}|$ denoting the population size. In the domain of single-objective optimization, the definition of the fitness difference is as follows:

\begin{equation}\label{eq4}
d(p_{i})=\max \left\{0, f(p_{i})-f^{\prime}\right\},
\end{equation}
where $i \in \{1, \ldots, |{\bar P}_{t}|\}$, $\mathit{p_{i}}$ is the $i$-th solution in $\bar P_t$, $\mathit{f(p_{i})}$ is the fitness value of the current solution, and $\mathit{f^{\prime}}$ is a fitness value desired to be obtained. Fitness difference can be interpreted as the discrepancy between the current solution, and the target solution.

During the optimization process, EAs generate offspring through stochastic operations, making the evolutionary process inherently random. This stochastic nature enables us to model the optimization dynamics as a stochastic process. Let $\mathit{\left\{s_{t}\right\}_{t=0}^{\infty}}$ denote a stochastic process on the probability space $\mathit{(\Omega, \mathcal{F}, \mathbb{P})}$, where $\mathit{\Omega}$ is the sample space, $\mathcal{F}$ is the sigma-algebra, and $\mathbb{P}$ is the probability measure. The sigma-algebra $\mathit{\mathcal{F}_{t}=\sigma\left(s_{0}, s_{1}, \ldots, s_{t}\right)}$ represents the natural filtration generated by the process up to generation $t$. The gain at the $\mathit{t}$-th generation is defined as follows \cite{huang2019experimental}:

\begin{equation}\label{eq5}
g_{t}=\varphi_{t}-\varphi_{t+1},
\end{equation}
where $\mathit{\varphi_{t}}$ denotes the smallest fitness difference in previous consecutive $\mathit{t}$ generations. Let $\mathit{G_{t}=\sigma\left(\varphi_{0}, \varphi_{1}, \ldots, \varphi_{t}\right)}$ be a sigma-algebra, the average gain at $\mathit{t}$ is defined as follows \cite{huang2019experimental}:

\begin{equation}\label{eq6}
\mathbb{E}\left(g_{t} \mid G_{t}\right)=\mathbb{E}\left(\varphi_{t}-\varphi_{t+1} \mid G_{t}\right).
\end{equation}

The gain denotes the difference in the best fitness between consecutive generations, reflecting the improvement achieved by the algorithm across generations. A larger gain indicates greater improvement within a single iteration, suggesting faster progress toward the target solution and consequently improved optimization efficiency.

Suppose that $\{\varphi_t\}_{t=0}^\infty$ is a stochastic process, where $\varphi_t \geq 0$ holds for any $t \geq 0$. Given a target precision $\varepsilon > 0$, the FHT of EAs is defined by
\begin{equation}\label{eq7}
T_\varepsilon = \min\{t | \varphi_t \leq \varepsilon\},
\end{equation}
where $T_\varepsilon$ represents the first time when the algorithm's fitness difference falls below the target precision $\varepsilon$ during the optimization process. In particular,
\begin{equation}\label{eq8}
T_0 = \min\{t | \varphi_t = 0\}.
\end{equation}
The EFHT of EAs is denoted by $\mathbb{E}(T_\varepsilon)$, representing the expected number of iterations required for EAs to reach the target solution.

Both the average gain model and drift analysis are used to analyze the running time of EAs, but they differ fundamentally in their analytical approaches, particularly regarding conditional expectations \cite{lengler2020drift,yushan2016first}. The average gain model calculates the expected improvement in the objective function by considering cumulative progress from generation 0 to $t$, with the sigma-algebra incorporating information from all previous generations up to $t$. It focuses on the expected gain per generation and derives convergence speed based on these accumulated gains. In contrast, drift analysis examines the drift process at the current generation $t$, analyzing how randomness and selection pressure influence the algorithm's immediate behavior. It employs recurrence relations to model cumulative drift and estimate convergence, emphasizing immediate changes rather than long-term cumulative progress.

\subsection{Proposed Gain Model for MOEAs}
To extend the average gain model to multi-objective optimization, we adopt IGD \cite{van1998evolutionary} as the performance indicator for measuring algorithmic progress. Let the reference point set be $P^* = \{\boldsymbol{v}_{1}^{*}, \boldsymbol{v}_2^{*}, \ldots, \boldsymbol{v}_{|P^*|}^{*}\}$ and the algorithm's solution set at generation $t$ be $P_t = \{\boldsymbol{u}_1^{(t)}, \boldsymbol{u}_2^{(t)}, \ldots, \boldsymbol{u}_{|P_t|}^{(t)}\}$, where $|\cdot|$ denotes the cardinality of a given set. The IGD is defined as:
\begin{equation}\label{IGD}
	\text{IGD}(P_t, P^*) = \frac{1}{|P^*|} \sum_{i=1}^{|P^*|} \min_{j \in \{1, 2, \ldots, |P_t|\}} \|\boldsymbol{v}_i^{*} - \boldsymbol{u}_j^{(t)}\|,
\end{equation}
where $\| \cdot \|$ denotes the Euclidean distance. IGD evaluates the quality of the obtained solution set by measuring the average distance from each reference point to its nearest solution, thereby reflecting both convergence toward and coverage of the true Pareto front.

The selection of IGD as our evaluation metric is motivated by two key considerations. First, to comprehensively capture the improvement of MOEAs at each iteration, we require a metric that simultaneously evaluates both convergence and diversity rather than focusing on a single aspect. While the hypervolume (HV) indicator \cite{zitzler2002multiobjective} is another comprehensive metric, it exhibits prohibitively high computational complexity, particularly for problems with more than two objectives. Second, our subsequent sampling methodology requires repeated independent evaluations at each generation, making computational efficiency crucial. Therefore, IGD provides an optimal balance between comprehensiveness and computational tractability.

Let $\psi_t$ denote the smallest IGD value achieved up to generation $t$, analogous to the fitness difference concept in single-objective EAs:
\begin{equation}\label{eq:best-igd}
	\psi_t = \min_{i \in \{0, 1, \ldots, t\}} \text{IGD}(P_i, P^*),
\end{equation}
where $\text{IGD}(P_i, P^*)$ is given by Eq. (\ref{IGD}). This quantity represents the minimum distance achieved from any population to the target Pareto front within the first $t+1$ generations.

\begin{definition}[IGD Gain]\label{def:2}
	Given a generation $t$, the IGD gain is defined as
	\begin{equation}\label{eq10}
		g_t = \psi_t - \psi_{t+1}.
	\end{equation}
\end{definition}

Let $\mathit{H}_{\mathrm{t}}=\sigma\left(\mathit{\psi}_{0}, \mathit{\psi}_{1}, \ldots, \mathit{\psi}_{t}\right)$ be the sigma-algebra generated by the IGD history up to generation $t$. The average IGD gain at generation $t$ is
\begin{equation}\label{eq11}
\mathbb{E}\left(g_{t} \mid H_{t}\right)=\mathbb{E}\left(\psi_{t}-\psi_{t+1} \mid H_{t}\right).
\end{equation}

The IGD gain quantifies the progress of MOEAs within a single iteration. A larger gain indicates more rapid reduction in the distance to the target Pareto front.

\begin{lemma}\label{lemma:1}
	Let $\{\eta_t\}_{t=0}^\infty$ be a stochastic process where $\eta_t \geq 0$ for any $t \geq 0$. Let $T_0^\eta = \min\{t | \eta_t = 0\}$. Assuming $\mathbb{E}(T_0^\eta) < +\infty$, if there exists $\alpha > 0$ such that $\mathbb{E}(\eta_t - \eta_{t+1} | H_t) \geq \alpha$ for any $t \geq 0$, then $\mathbb{E}(T_0^\eta | \eta_0) \leq (\eta_0 / \alpha)$.
\end{lemma}

Lemma \ref{lemma:1} \cite{yushan2016first} provides the theoretical foundation for Theorem \ref{theorem:1}.

\begin{theorem}\label{theorem:1}
	Let $\mathit{\{\psi_{t}\}_{t=0}^{\infty}}$ be a stochastic process in MOEAs with $\psi_t \geq 0$ for all $t \geq 0$. Let $h: (0, \psi_{0}] \rightarrow \mathbb{R}^{+}$ be a monotonically increasing and continuous function. Given a target precision $\varepsilon > 0$, if 
	\[
	\mathbb{E}\left(\psi_{t} - \psi_{t+1} \mid H_{t}\right) \geq h(\psi_t)
	\]
	holds whenever $\psi_t > \varepsilon$, then the expected time $\mathbb{E}(T_{\varepsilon} \mid \psi_{0})$ for the IGD value to reach $\varepsilon$ satisfies the following upper bound:  
	\begin{equation}\label{eq12}
	\mathbb{E}(T_{\varepsilon} \mid \psi_{0}) \leq 1 + \int_{\varepsilon}^{\psi_{0}} \frac{1}{h(z)} dz.
	\end{equation}
\end{theorem}

\begin{proof}
	Define the auxiliary function
	\[
	l(z) =
	\begin{cases} 
		0, & z \leq \varepsilon, \\
		\int_\varepsilon^z \frac{1}{h(w)} \, dw + 1, & z > \varepsilon.
	\end{cases}
	\]
	
	We consider two cases:
	
	\textbf{Case 1:} If $\psi_t > \varepsilon$ and $\psi_{t+1} \leq \varepsilon$, then
	\[
	l(\psi_t) - l(\psi_{t+1}) = 1 + \int_\varepsilon^{\psi_t} \frac{1}{h(w)} \, dw \geq 1.
	\]
	Thus, $\mathbb{E}(l(\psi_t) - l(\psi_{t+1}) \mid H_t) \geq 1$.
	
	\textbf{Case 2:} If $\psi_t > \varepsilon$ and $\psi_{t+1} > \varepsilon$, then
	\[
	l(\psi_t) - l(\psi_{t+1}) = \int_{\psi_{t+1}}^{\psi_t} \frac{1}{h(w)} \, dw \geq \frac{\psi_t - \psi_{t+1}}{h(\psi_t)}.
	\]
	Therefore,
	\[
	\mathbb{E}(l(\psi_t) - l(\psi_{t+1}) \mid H_t) \geq \frac{\mathbb{E}(\psi_t - \psi_{t+1} \mid H_t)}{h(\psi_t)} \geq 1.
	\]
	
	In both cases, $\mathbb{E}(l(\psi_t) - l(\psi_{t+1}) \mid H_t) \geq 1$ whenever $\psi_t > \varepsilon$. Let $T^l_0 = \min\{t \mid l(\psi_t) = 0\}$ denote the FHT of the process $\{l(\psi_t)\}_{t=0}^\infty$. Since $T_\varepsilon = \min\{t \mid \psi_t \leq \varepsilon\} = \min\{t \mid l(\psi_t) = 0\}$, applying Lemma \ref{lemma:1} yields
	\[
	\mathbb{E}(T_\varepsilon \mid \psi_0) = \mathbb{E}(T^l_0 \mid l(\psi_0)) \leq l(\psi_0) = 1 + \int_\varepsilon^{\psi_0} \frac{1}{h(z)} \, dx.
	\]
\end{proof}

Theorem \ref{theorem:1} establishes the theoretical foundation for using average IGD gain to analyze upper bounds on EFHT of MOEAs. The sequence $\mathit{\left\{\psi_{t}\right\}_{t=0}^{\infty}}$ represents the progression of historical minimum IGD values that monotonically decrease during the optimization process. However, the primary challenge in applying this theoretical framework lies in analytically deriving the probability distribution of the IGD gain $\mathbb{E}\left(\psi_{t}-\psi_{t+1} \mid H_{t}\right)$. 
To overcome this challenge, we propose an experimental approach that estimates $\mathbb{E}\left(\psi_{t}-\psi_{t+1} \mid H_{t}\right)$ through statistical sampling, followed by surface fitting techniques to determine the function $h(\psi_t)$. When the fitted function $h(\psi_t)$ satisfies the conditions specified in Theorem \ref{theorem:1}, it enables the computation of upper bounds on EFHT $\mathbb{E}\left(T_{\varepsilon} \mid \psi_{0}\right)$. The detailed implementation procedures of this experimental methodology are presented in the following section.

\section{Experimental Approach for Running-Time Estimation Based on IGD-Gain Model}\label{Approach}

\begin{figure*}[t]
	\centering
	\includegraphics[width=0.8\textwidth]{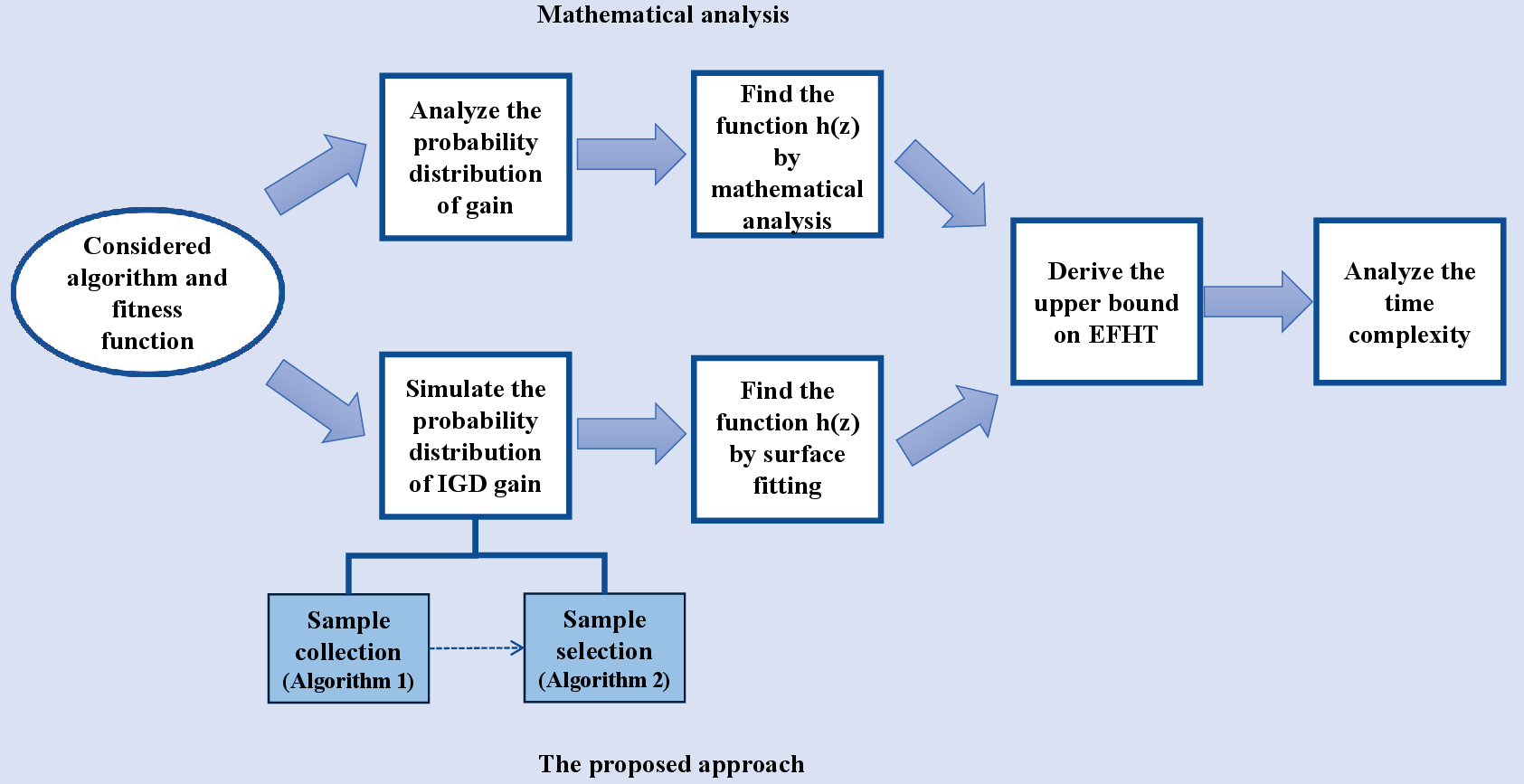}
	\caption{Comparison between mathematical derivation and the proposed experimental estimation approach.}
	\label{fig1}
\end{figure*}

This section presents an experimental approach for estimating the running time of MOEAs in numerical optimization. The overall framework of our proposed methodology is illustrated in Fig. \ref{fig1}. Our approach consists of four main phases: First, we collect numerical data regarding IGD values and their corresponding average gains through systematic statistical experiments. Second, we employ the empirical distribution function derived from the collected data to approximate the probability distribution of the IGD gain. Third, to enhance the accuracy of this approximation, we introduce an adaptive sampling method that dynamically adjusts sampling density based on the characteristics of the gain distribution. Finally, we apply surface fitting techniques to transform the experimental results into a function $h(\psi_t)$ that satisfies the conditions specified in Theorem \ref{theorem:1}. Based on the fitted surface, we utilize Theorem \ref{theorem:1} to compute the estimated upper bound on the EFHT of the considered MOEAs.

\subsection{Statistical Estimation of IGD Gain Distribution}
A crucial component of our proposed methodology is the estimation of $\mathbb{E}\left(T_{\varepsilon} \mid \psi_{0}\right)$. Our approach is theoretically grounded in the Glivenko-Cantelli theorem \cite{tucker1959generalization}, a fundamental result in probability theory that establishes the convergence relationship between empirical and true distribution functions. The theorem states that as the sample size increases, the empirical distribution function converges almost surely to the true distribution function. Formally, for sufficiently large sample sizes, the empirical distribution function approaches the true distribution function with probability 1 \cite{tucker1959generalization}.

Let $K$ denote the sample size, and let $X_1, X_2, \ldots, X_K$ represent the generated samples, ordered such that $X_1 \leq X_2 \leq \ldots \leq X_K$. Suppose that $T_{\varepsilon} \mid \psi_{0} \sim Q(r)$, where $r = \psi_t - \psi_{t+1}$, and that the empirical distribution function $Q_K(r)$ is simulated by the statistical experiment based on $H_t$. The true distribution $Q(r)$ can be estimated by $Q_K(r)$ when $K$ is sufficiently large. The empirical distribution function is defined as
\begin{equation}\label{eq13}
	Q_{K}(r)=\left\{\begin{array}{ll}
		0, & r<X_{1} \\
		\frac{i}{K}, & X_{i} \leq r<X_{i+1}, \quad i=1,2, \ldots, K-1 \\
		1, & r \geq X_{K}
	\end{array}\right.
\end{equation}

By the Glivenko-Cantelli theorem, when $K$ is sufficiently large, the empirical distribution function $Q_K(r)$ converges to the true distribution function $Q(r)$. Consequently, the sample mean of $X_1, X_2, \ldots, X_K$ approaches the expectation of $\psi_t - \psi_{t+1}$. Therefore, we obtain
\begin{equation}\label{eq14}
	\mathbb{E}\left(\psi_{t}-\psi_{t+1} \mid H_{t}\right) \approx \mathbb{E}(X_1, X_2, \ldots, X_K).
\end{equation}

\subsection{Adaptive Sampling Method}

This subsection presents our adaptive sampling methodology, which forms the core experimental component for estimating IGD gain distributions. The methodology consists of two algorithms: Algorithm \ref{alg:1} performs systematic data collection during MOEA execution to gather IGD values and corresponding average gains across different problem dimensions, while Algorithm \ref{alg:2} implements intelligent sample selection and preprocessing techniques to ensure effective surface fitting for running time estimation.

\begin{algorithm}[h]
	\caption{IGD-Gain Data Collection from MOEAs}
	\label{alg:1}
	\KwIn{A set of problem dimensions $Ns = \{n_1, n_2, \ldots, n_j\}$, fixed sample size $K$ \tcp{$j$ is the cardinality of the set of problem dimensions}}
	\KwOut{Sample points $S = \emptyset$}
	
	\For{each $n \in Ns$}{
		Initialize population $P_0$ randomly\;
		$t \leftarrow 0$\;
		\While{termination criteria is not satisfied}{
			\For{$i \leftarrow 1$ \KwTo $K$}{
				Generate offspring population $P'$ \;\tcp{according to the procedure of the considered MOEA}
                
				Evaluate each individual of parent population $P$ and offspring population $P'$\;
				Record minimum IGD value $\psi_{t}^{(i)}$ and corresponding gain $g_{t}^{(i)}$\;
			}
			Compute average gain $\overline{g}_t$ (according to Eq. (\ref{avg})) and minimum IGD value $\psi_t$ for generation $t$\;
			Perform environmental selection to form next generation population\;
			$S \leftarrow S \cup \{(\psi_t, \overline{g}_t)\}$\;
			$t \leftarrow t + 1$\;
		}
	}
	\Return{$S$}\;
\end{algorithm}

Algorithm \ref{alg:1} focuses on comprehensive data collection during the optimization process of MOEAs. The algorithm operates by running the considered MOEA while simultaneously recording IGD values and computing average gains without modifying the original algorithmic procedures. During each generation $t$, we conduct $K$ independent sampling runs to collect the $i$-th observed minimum IGD value $\psi_{t}^{(i)}$ and the corresponding gain $g_{t}^{(i)}$. The average gain at generation $t$ is calculated (Line 10) as
\begin{equation}\label{avg}
	\overline{g}_{t}
	\;=\;
	\frac{1}{K}
	\sum_{i=1}^{K} g_{t}^{(i)} \,.
\end{equation}
By systematically collecting each generation's minimum IGD and average gain values $\bigl(\psi_{t}, \overline{g}_{t}\bigr)$ and appending them to the sampling set $S$ (Line 12), we construct a comprehensive dataset that captures the optimization dynamics across different problem dimensions.

Algorithm \ref{alg:2} implements an adaptive sample selection strategy designed to address two critical challenges in surface fitting. First, utilizing excessive sample points would incur prohibitive computational costs during surface fitting procedures. Second, fitting large datasets necessitates complex surface models, resulting in intricate functional expressions that compromise both interpretability and computational efficiency. To address these challenges, we introduce an adaptive parameter $M$ that determines the total number of sample points for each dimension, where $M$ is positively correlated with the problem dimension $n$ (Line 2). This correlation reflects the observation that higher-dimensional problems typically require more iterations for convergence, necessitating correspondingly denser sampling to accurately capture optimization dynamics. The algorithm first excludes outlier sample points with zero average gain (Line 3), as such points indicate stagnation periods where the algorithm fails to achieve IGD improvements. Subsequently, we apply Locally Estimated Scatterplot Smoothing (LOESS) \cite{cleveland1979robust} for noise reduction (Line 4), which fits simple regression models to localized data subsets, effectively capturing underlying trends while filtering high-frequency noise. The sample selection strategy subsequently identifies samples with uniformly distributed IGD values (Lines 5-9) to facilitate effective surface fitting. Finally, we apply an adaptive scaling parameter $\lambda$ to the selected gain samples (Lines 10-13). Although the denoising process yields smoother data and more stable fitting results, it inevitably removes peak values representing significant optimization breakthroughs. The adaptive scaling parameter compensates for this information loss by proportionally amplifying gain values, ensuring that the estimated upper bounds remain meaningful and not overly conservative.

\begin{algorithm}[h]
	\caption{Adaptive Sample Point Selection}
	\label{alg:2}
	\KwIn{Problem dimensions $Ns = \{n_1, n_2, \ldots, n_j\}$, collected sample points $S$}
	\KwOut{Selected sample points $S_M = \emptyset$}
	
	\ForEach{$n \in Ns$}{
		Set adaptive sampling size $M \leftarrow 2 \cdot n$\;
		Remove sample points where average gain equals zero\;
		Apply LOESS denoising to $S$\;
		Generate arithmetic sequence $\xi$ with $M$ elements uniformly distributed in $[\psi_{min}, \psi_{max}]$\;
		\For{$i \leftarrow 1$ \KwTo $M$}{
			Find the sample point $(\psi^*, g^*) \in S$ such that $|\psi^* - \xi_i|$ is minimized\; \tcp{$|\cdot|$ denotes the absolute value}
			$S_M \leftarrow S_M \cup \{(\psi^*, g^*)\}$\;
		}
		Compute adaptive scaling parameter $\lambda \leftarrow \frac{g_{max}}{g_{mean}}$\; \tcp{$g_{\max}$ and $g_{mean}$ denote the maximum and mean average gain values in $S_M$, respectively}
		\ForEach{$(\psi, g) \in S_M$}{
			$g \leftarrow g \cdot \lambda$\;
		}
	}
	\Return{$S_M$}\; 
\end{algorithm}

\subsection{Surface Fitting and Complexity Analysis}

According to Theorem \ref{theorem:1}, after collecting and preprocessing the sample data, we need to determine a function $h(\psi_t)$ to derive the upper bound on the EFHT of MOEAs. Surface fitting serves as a powerful mathematical tool for transforming discrete experimental data into continuous analytical expressions. This technique enables us to construct smooth surfaces from scattered data points, effectively capturing the underlying relationships between IGD values, average gains, and problem dimensions, thereby facilitating the derivation of function $h(\psi_t)$ that satisfies the conditions specified in Theorem \ref{theorem:1}.

The selection of an appropriate functional form for surface fitting is guided by empirical observations of the optimization dynamics in MOEAs. Our experimental analysis reveals that the average gain typically exhibits a positive correlation with IGD values, reflecting the intuitive notion that larger IGD values (indicating greater distance from the Pareto front) provide more room for improvement and thus enable larger gains. Conversely, the average gain demonstrates a negative correlation with problem dimensions, consistent with the increased optimization difficulty in higher-dimensional spaces. Based on these empirical patterns and following methodologies similar to those described in \cite{huang2019experimental}, we adopt a power-law functional form that captures these relationships effectively.

The mathematical expression employed for surface fitting is given by:
\begin{equation}
	\label{eq15}
	f(\psi, n)=\frac{a \times \psi^{b}}{c \times n^{d}}, \quad a, c, d>0, \quad b \geq 1,
\end{equation}
where $n$ represents the problem dimension, and $a$, $b$, $c$, $d$ are parameters to be determined through the fitting process. This functional form systematically captures the observed relationships: the numerator $a \times \psi^b$ models the positive correlation between gain and IGD value with the constraint $b \geq 1$ ensuring monotonic behavior, while the denominator $c \times n^d$ represents the inverse relationship with problem dimension. The power-law structure provides sufficient flexibility to accommodate various optimization scenarios while maintaining mathematical tractability for subsequent analysis. Once the parameter values are determined through least-squares fitting or similar optimization techniques, the complete surface representation enables the direct application of Theorem \ref{theorem:1} for EFHT upper bound estimation.

Assuming that the continuous function $h(z)$ in Theorem \ref{theorem:1} conforms to the structure of Eq. (\ref{eq15}), we derive the following corollary to directly estimate the running time complexity of the considered MOEAs.

\begin{corollary}\label{colry:1}
	Given sample data points $(n_i, \psi_i, g_i)$ for $i = 1, 2, \ldots, \kappa$, the initial IGD value $X_0$, and target precision $\varepsilon > 0$, where $g_i$ denotes the average gain associated with IGD value $\psi_i$, and $\kappa$ is the total number of sampling points across all problem dimensions. Through surface fitting of these discrete sample points, we obtain a continuous function $f(\psi, n) = ({a \times \psi^{b}})/({c \times n^{d}})$ with parameters $a,c,d>0$ and $b \geq 1$, where $\psi$ represents the continuous IGD variable. Assuming that this fitted surface provides a lower bound for all observed gain values, the expected first hitting time satisfies:
	\begin{enumerate}
		\item If $b=1$, then $\mathbb{E}\left(T_{\varepsilon} \mid X_{0}\right) \in O\left(n^{d} \ln \left(\frac{X_{0}}{\varepsilon}\right)\right)$;
		
		\item If $b>1$, then $\mathbb{E}\left(T_{\varepsilon} \mid X_{0}\right) \in O\left(n^{d} \varepsilon^{-(b-1)}\right)$.
	\end{enumerate}
\end{corollary}

\begin{proof}
	When problem dimension $n$ is fixed, $f(\psi, n)$ can be treated as a univariate function of $\psi$. Given the constraints $a, c > 0$ and $b \geq 1$, the function $f(\psi, n)$ is strictly positive and monotonically increasing over the interval $(0, +\infty)$, thereby satisfying the conditions specified in Theorem \ref{theorem:1}. Applying Theorem \ref{theorem:1}, the EFHT is bounded above by $1 + \int_{\varepsilon}^{X_0} ({c \cdot n^d}) / ({a \cdot \psi^b}) \, d\psi$.
	\begin{enumerate}
		\item For the case $b = 1$: Since $a, c > 0$, we have
		\begin{align}
			1 + \int_{\varepsilon}^{X_0} \frac{c \cdot n^d}{a \cdot \psi} \, d\psi &= 1 + \frac{c \cdot n^d}{a} \int_{\varepsilon}^{X_0} \frac{1}{\psi} \, d\psi \\
			&= 1 + \frac{c \cdot n^d}{a} \ln(\psi) \big|_{\varepsilon}^{X_0} \\
			&= 1 + \frac{c \cdot n^d}{a} \ln\left(\frac{X_0}{\varepsilon}\right)
		\end{align}
		Therefore, $\mathbb{E}(T_{\varepsilon} \mid X_0) \in O\left(n^d \ln\left(\frac{X_0}{\varepsilon}\right)\right)$.
		
		\item For the case $b > 1$: Since $a, c > 0$, $b > 1$, and typically $\varepsilon \ll X_0$, we have
		\begin{align}
			1 + \int_{\varepsilon}^{X_0} \frac{c \cdot n^d}{a \cdot \psi^b} \, d\psi &= 1 + \frac{c \cdot n^d}{a} \int_{\varepsilon}^{X_0} \psi^{-b} \, d\psi \\
			&= 1 + \frac{c \cdot n^d}{a} \left[ \frac{\psi^{1-b}}{1-b} \right]_{\varepsilon}^{X_0} \\
			&= 1 + \frac{c \cdot n^d}{a(1-b)} \left( X_0^{1-b} - \varepsilon^{1-b} \right)
		\end{align}
		Since $b > 1$ and $\varepsilon \ll X_0$, the term $\varepsilon^{1-b}$ dominates, yielding $\mathbb{E}(T_{\varepsilon} \mid X_0) \in O\left(n^d \varepsilon^{-(b-1)}\right)$.
	\end{enumerate}
\end{proof}

Throughout the experimental studies presented in this paper, we consistently employ Eq. (\ref{eq15}) for surface fitting. Consequently, the experimental results can be directly applied to derive upper bounds on running time complexity using the analytical framework established in Corollary \ref{colry:1}.

\section{Experimental Study}\label{Experiment}
This section presents experimental validation of our proposed approach for estimating upper bounds on the EFHT of MOEAs. We evaluate five representative algorithms: NSGA-II \cite{deb2002fast}, MOEA/D \cite{zhang2007moea}, AR-MOEA \cite{tian2017indicator}, PREA \cite{yuan2020investigating}, and AGEMOEA-II \cite{panichella2022improved}, covering the three main MOEA categories: dominance-based, decomposition-based, and indicator-based approaches.
The experimental testbed uses benchmark problems from the ZDT \cite{zitzler2000comparison} and DTLZ \cite{deb2002scalable} test suites: ZDT1, ZDT2, ZDT3, ZDT4, ZDT6, DTLZ1, DTLZ2, DTLZ3, DTLZ5, and DTLZ6. For most problems, we employ dimensions $n \in \{5, 10, 15, 20, 25, 30\}$ following \cite{huang2019experimental}. For ZDT4 and ZDT6, we use $n \in \{2, 4, 6, 8, 10\}$ due to their variable domain constraints. The statistical sampling uses $K = 100$. Additionally, we performed a comparative validation against existing theoretical results. All implementations are sourced from the PlatEMO platform \cite{PlatEMO}.

\captionsetup[table*]{labelformat=simple, labelsep=newline, textfont=sc}
\renewcommand\arraystretch{1.5}
\begin{table*}[b]
    \caption{
        \footnotesize \MakeUppercase Comparison of Estimated and Numerical Experiment Results for Running Time of NSGA-\uppercase\expandafter{\romannumeral2}.
    }
    \begin{center}
    \resizebox{0.8\textwidth}{!}{
        \begin{tabular}{c c c c c c}
            \toprule
            ~~~~~~ & ~~~Estimated Running Time~~~ & ~~~St.D.~~~ & ~~~Mean~~~ & ~~~Estimation Results~~~ & ~~~Deviation($R^2$)~~~ \\
            \midrule
            ZDT1   & $1.423 \times n^{1.234} \ln \left(\frac{X_{0}}{\varepsilon}\right)+1$ & 5.07E+02 & 5.78E+03 & 2.00E+4 & 0.533 \\ 
            ZDT2   & $4.637 \times n^{1.015} \ln \left(\frac{X_{0}}{\varepsilon}\right)+1$ & 1.15E+03 & 6.70E+03 & 4.00E+4 & 0.792 \\ 
            ZDT3   & $5.477 \times n^{0.929} \ln \left(\frac{X_{0}}{\varepsilon}\right)+1$ & 6.57E+03 & 7.36E+03 & 3.09E+04 & 0.904 \\ 
            ZDT4   & $30.000 \times n^{0.583} \ln \left(\frac{X_{0}}{\varepsilon}\right)+1$ & 3.12E+03 & 1.84E+04 & 9.95E+04 & 0.867 \\
            ZDT6  & $15.290 \times n^{1.045} \ln \left(\frac{X_{0}}{\varepsilon}\right)+1$ & 1.49E+03 & 1.44E+04 & \ 1.09E+05 & 0.120 \\ 
            DTLZ1  & $30.000 \times n^{0.573} \ln \left(\frac{X_{0}}{\varepsilon}\right)+1$ & 1.70E+04 & 5.41E+04 & 1.15E+05 & 0.909 \\ 
            DTLZ2  & $5.477 \times n^{0.921} \ln \left(\frac{X_{0}}{\varepsilon}\right)+1$ & 7.33E+03 & 1.69E+04 & 1.79E+04 & 0.183 \\ 
            DTLZ3  & $30.000 \times n^{1.333} \ln \left(\frac{X_{0}}{\varepsilon}\right)+1$ & 1.50E+04 & 6.34E+04 & 1.00E+06 & \textbf{-1.714} \\ 
            DTLZ5  & $5.477 \times n^{0.950} \ln \left(\frac{X_{0}}{\varepsilon}\right)+1$ & 5.21E+02 & 5.22E+03 & 2.73E+04 &  0.448 \\ 
            DTLZ6  & $30.000 \times n^{0.808} \ln \left(\frac{X_{0}}{\varepsilon}\right)+1$ & 1.86E+03 & 8.32E+03 & 1.85E+05 & \textbf{-0.654} \\ 
            \bottomrule
        \end{tabular}
        }
        \label{t1}
    \end{center}
\end{table*}

\captionsetup[table*]{labelformat=simple, labelsep=newline, textfont=sc}
\renewcommand\arraystretch{1.5}
\begin{table*}[b]
    \caption{
        \footnotesize \MakeUppercase Comparison of Estimated and Numerical Experiment Results for Running Time of MOEA/D.
    }
    \begin{center}
    \resizebox{0.8\textwidth}{!}{
        \begin{tabular}{c c c c c c}
            \toprule
            ~~~~~~ & ~~~Estimated Running Time~~~ & ~~~St.D.~~~ & ~~~Mean~~~ & ~~~Estimation Results~~~ & ~~~Deviation($R^2$)~~~ \\
            \midrule
            ZDT1   & $5.477 \times n^{0.794} \ln \left(\frac{X_{0}}{\varepsilon}\right)+1$ & 1.65E+03 & 7.62E+03 & 2.26E+04 & 0.317 \\
            ZDT2   & $5.477 \times n^{0.862} \ln \left(\frac{X_{0}}{\varepsilon}\right)+1$ & 1.70E+03 & 7.17E+03 & 2.90E+04 & 0.411 \\
            ZDT3   & $5.477 \times n^{0.599} \ln \left(\frac{X_{0}}{\varepsilon}\right)+1$ & 5.59E+03 & 1.18E+04 & 1.53E+04 & 0.721  \\ 
            ZDT4   & $8.272 \times n^{1.169} \ln \left(\frac{X_{0}}{\varepsilon}\right)+1$ & 3.17E+03 & 2.32E+04 & 7.21E+04 & 0.579 \\ 
            ZDT6   & $5.477 \times n^{1.166} \ln \left(\frac{X_{0}}{\varepsilon}\right)+1$ & 9.45E+02 & 1.05E+04 & 5.06E+04 & 0.626 \\ 
            DTLZ1  & $11.825 \times n^{1.164} \ln \left(\frac{X_{0}}{\varepsilon}\right)+1$ & 7.15E+03 & 3.73E+04 & 2.05E+05 & 0.101 \\ 
            DTLZ2  & $5.477 \times n^{0.892} \ln \left(\frac{X_{0}}{\varepsilon}\right)+1$ & 2.33E+03 & 7.93E+03 & 1.11E+04 & 0.379 \\ 
            DTLZ3  & $28.271 \times n^{0.627} \ln \left(\frac{X_{0}}{\varepsilon}\right)+1$ & 7.96E+03 & 4.56E+04 & 1.28E+05 & 0.739 \\ 
            DTLZ5  & $5.477 \times n^{0.766} \ln \left(\frac{X_{0}}{\varepsilon}\right)+1$ & 7.61E+02 & 3.65E+03 & 1.42E+04 & 0.572 \\ 
            DTLZ6  & $3.817 \times n^{1.361} \ln \left(\frac{X_{0}}{\varepsilon}\right)+1$ & 2.03E+03 & 7.92E+03 & 9.85E+04 & \textbf{-1.83} \\ 
            \bottomrule
        \end{tabular}
        }
        \label{t2}
    \end{center}
\end{table*}

\captionsetup[table*]{labelformat=simple, labelsep=newline, textfont=sc}
\renewcommand\arraystretch{1.5}
\begin{table*}[!b]
    \caption{
        \footnotesize \MakeUppercase Comparison of Estimated and Numerical Experiment Results for Running Time of AR-MOEA.
    }
    \begin{center}
    \resizebox{0.8\textwidth}{!}{    
        \begin{tabular}{c c c c c c}
            \toprule
            ~~~~~~ & ~~~Estimated Running Time~~~ & ~~~St.D.~~~ & ~~~Mean~~~ & ~~~Estimation Results~~~ & ~~~Deviation($R^2$)~~~ \\
            \midrule
            ZDT1   & $5.477 \times n^{0.727} \ln \left(\frac{X_{0}}{\varepsilon}\right)+1$ & 4.08E+03 & 1.37E+04 & 1.93E+04 & 0.351 \\ 
            ZDT2   & $30.000 \times n^{0.377} \ln \left(\frac{X_{0}}{\varepsilon}\right)+1$ & 2.79E+03 & 1.64E+4E & 4.63E+04 & 0.459 \\ 
            ZDT3   & $5.480 \times n^{0.934} \ln \left(\frac{X_{0}}{\varepsilon}\right)+1$ & 1.73E+04 & 2.40E+04 & 3.65E+04 & 0.618 \\
            ZDT4   & $30.000 \times n^{0.817} \ln \left(\frac{X_{0}}{\varepsilon}\right)+1$ & 5.51E+03 & 3.20E+04 & 1.75E+05 & 0.637 \\
            ZDT6   & $18.333 \times n^{1.488} \ln \left(\frac{X_{0}}{\varepsilon}\right)+1$ & 2.25E+03 & 1.93E+04 & 3.65E+05 & \textbf{-2.09} \\
            DTLZ1  & $23.636 \times n^{0.709} \ln \left(\frac{X_{0}}{\varepsilon}\right)+1$ & 6.82E+03 & 4.87E+04 & 1.20E+05 & 0.662 \\
            DTLZ2  & $5.477 \times n^{0.700} \ln \left(\frac{X_{0}}{\varepsilon}\right)+1$ & 6.54E+03 & 6.10E+03 & 8.36E+04 & 0.432 \\
            DTLZ3  & $30.000 \times n^{1.042} \ln \left(\frac{X_{0}}{\varepsilon}\right)+1$ &  9.08E+03 & 6.52E+04 & 4.47E+05 & 0.229 \\
            DTLZ5  & $5.451 \times n^{0.521} \ln \left(\frac{X_{0}}{\varepsilon}\right)+1$ & 5.61E+03 & 6.07E+03 & 8.40E+03 & 0.490 \\
            DTLZ6  & $26.450 \times n^{0.915} \ln \left(\frac{X_{0}}{\varepsilon}\right)+1$ & 1.97E+03 & 8.55E+03 & 2.17E+05 & \textbf{-1.25} \\ 
            \bottomrule
        \end{tabular}
        }
        \label{t3}
    \end{center}
\end{table*}

\captionsetup[table*]{labelformat=simple, labelsep=newline, textfont=sc}
\renewcommand\arraystretch{1.5}
\begin{table*}[t]
    \caption{
        \footnotesize \MakeUppercase Comparison of Estimated and Numerical Experiment Results for Running Time of AGE-MOEA-\uppercase\expandafter{\romannumeral2}.
    }
    \begin{center}
    \resizebox{0.8\textwidth}{!}{    
        \begin{tabular}{c c c c c c}
            \toprule
            ~~~~~~ & ~~~Estimated Running Time~~~ & ~~~St.D.~~~ & ~~~Mean~~~ & ~~~Estimation Results~~~ & ~~~Deviation($R^2$)~~~ \\
            \midrule
            ZDT1   & $5.477 \times n^{0.966} \ln \left(\frac{X_{0}}{\varepsilon}\right)+1$ & 4.18E+02 & 5.14E+03 & 3.64E+04 & 0.451 \\
            ZDT2   & $5.477 \times n^{1.096} \ln \left(\frac{X_{0}}{\varepsilon}\right)+1$ & 1.29E+03 & 6.42E+03 & 5.68E+04 & 0.253 \\
            ZDT3   & $5.477 \times n^{1.006} \ln \left(\frac{X_{0}}{\varepsilon}\right)+1$ & 3.77E+03 & 6.18E+03 & 3.90E+04 & 0.336 \\
            ZDT4   & $9.144 \times n^{1.275} \ln \left(\frac{X_{0}}{\varepsilon}\right)+1$ & 2.45E+03 & 1.67E+04 & 1.46E+05 & 0.753 \\
            ZDT6   & $9.709 \times n^{1.276} \ln \left(\frac{X_{0}}{\varepsilon}\right)+1$ & 1.33E+03 & 1.26E+04 & 1.15E+05 & 0.472 \\
            DTLZ1  & $18.418 \times n^{0.747} \ln \left(\frac{X_{0}}{\varepsilon}\right)+1$ & 8.47E+03 & 3.55E+04 & 1.15E+05 & 0.723 \\
            DTLZ2  & $5.441 \times n^{0.697} \ln \left(\frac{X_{0}}{\varepsilon}\right)+1$ & 3.91E+03 & 3.80E+03 & 8.31E+03 & 0.319 \\
            DTLZ3  & $10.694 \times n^{1.111} \ln \left(\frac{X_{0}}{\varepsilon}\right)+1$ & 9.24E+03 & 4.97E+04 & 1.94E+05 & 0.610 \\
            DTLZ5  & $5.477 \times n^{1.187} \ln \left(\frac{X_{0}}{\varepsilon}\right)+1$ & 5.03E+02 & 4.61E+03 & 4.75E+04 & \textbf{-0.910} \\
            DTLZ6  & $5.477 \times n^{1.319} \ln \left(\frac{X_{0}}{\varepsilon}\right)+1$ & 1.36E+03 & 7.63E+03 & 1.31E+05 & \textbf{-0.416} \\
            \bottomrule
        \end{tabular}
        }
        \label{t4}
    \end{center}
\end{table*}

\captionsetup[table*]{labelformat=simple, labelsep=newline, textfont=sc}
\renewcommand\arraystretch{1.5}
\begin{table*}[!t]
    \caption{
        \footnotesize \MakeUppercase Comparison of Estimated and Numerical Experiment Results for Running Time of PREA.
    }
    \begin{center}
    \resizebox{0.8\textwidth}{!}{    
            \begin{tabular}{c c c c c c}
            \toprule
                ~~~~~~ & ~~~Estimated Running Time~~~ & ~~~St.D.~~~ & ~~~Mean~~~ & ~~~Estimation Results~~~ & ~~~Deviation($R^2$)~~~ \\
            \midrule
            ZDT1   & $4.734 \times n^{1.046} \ln \left(\frac{X_{0}}{\varepsilon}\right)+1$ & 4.0765E+03 & 1.35E+04 & 3.88E+04 & 0.760 \\ 
            ZDT2   & $5.701 \times n^{1.103} \ln \left(\frac{X_{0}}{\varepsilon}\right)+1$ & 2.85E+03 & 1.73E+04 & 6.03E+04 & 0.410 \\ 
            ZDT3   & $5.477 \times n^{1.154} \ln \left(\frac{X_{0}}{\varepsilon}\right)+1$ & 1.87E+04 & 2.33E+04 & 5.71E+04 & 0.445 \\ 
            ZDT4   & $30.000 \times n^{0.719} \ln \left(\frac{X_{0}}{\varepsilon}\right)+1$ & 5.26E+03 & 3.20E+04 & 1.33E+05 & 0.768 \\ 
            ZDT6   & $10.166 \times n^{1.328} \ln \left(\frac{X_{0}}{\varepsilon}\right)+1$ & 2.02E+03 & 1.92E+04 & 1.38E+05 & 0.157 \\ 
            DTLZ1  & $9.186 \times n^{1.172} \ln \left(\frac{X_{0}}{\varepsilon}\right)+1$ & 8.19E+03 & 4.86E+04 & 1.79E+05 & 0.567 \\ 
            DTLZ2  & $5.477 \times n^{0.748} \ln \left(\frac{X_{0}}{\varepsilon}\right)+1$ & 6.07E+02 & 6.06E+03 & 9.79E+03 & 0.301 \\ 
            DTLZ3  & $30.00 \times n^{0.860} \ln \left(\frac{X_{0}}{\varepsilon}\right)+1$ & 1.03E+04 & 6.64E+04 & 2.77E+05 & 0.394 \\ 
            DTLZ5  & $5.477 \times n^{0.771} \ln \left(\frac{X_{0}}{\varepsilon}\right)+1$ & 8.52E+02 & 6.31E+03 & 1.65E+04 & 0.516 \\ 
            DTLZ6  & $27.902 \times n^{0.834} \ln \left(\frac{X_{0}}{\varepsilon}\right)+1$ & 2.20E+03 & 1.25E+04 & 1.825E+05 & \textbf{-0.349} \\ 
            \bottomrule
            \end{tabular}
        }
        \label{t5}
    \end{center}
\end{table*}

The experimental study is designed to systematically validate three fundamental aspects of our proposed methodology:
\begin{enumerate}
	\item Demonstrate that our approach provides mathematically sound and practically meaningful upper bounds on the EFHT of MOEAs, where the estimated bounds consistently exceed or equal the empirically observed runtime performance across different problem-algorithm combinations.
	
	\item Establish that our framework operates effectively across diverse MOEA paradigms and problem characteristics without requiring algorithm-specific modifications or problem-dependent simplifications, thereby confirming its broad applicability in numerical multi-objective optimization.
	
	\item Verify that our adaptive sampling methodology produces consistent and reliable estimates of IGD gain distributions, ensuring the robustness of the overall estimation framework against the inherent stochasticity of evolutionary optimization processes.
\end{enumerate}

\subsection{Upper Bound Estimation on EFHT of MOEAs}
\label{chap5:1}
This subsection presents both theoretical and empirical validations of the proposed estimation approach, with results consistently supporting its reliability and effectiveness.

\subsubsection{Theoretical Validation}
It is worth noting that, to date, no suitable theoretical upper bounds exist for MOEAs in the context of continuous optimization problems. Therefore, to validate our estimation approach, we draw a comparison to the well-studied discrete domain. Specifically, in the binary OneMinMax problem \cite{6793344}, each bitstring $x \in \{0,1\}^n$ yields objectives $f_1 = |x|$ and $f_2 = n - |x|$, and it has been proven that NSGA-II requires $\mathcal{O}(n \log n)$ generations to reach the full Pareto front under standard settings \cite{bian2022running}.

To create a continuous analogue, we extend the OneMinMax problem by allowing each decision variable $x_i$ to take real values in $[0,1]$, and define the two objectives as $f_1 = \sum_{i=1}^n x_i$ and $f_2 = n - \sum_{i=1}^n x_i$. This formulation preserves the original problem’s structural symmetry while embedding it into a continuous, real-coded search space.

By applying the proposed estimation approach to the continuous version of OneMinMax, we empirically obtain an upper bound on EFHT for NSGA‑II given by
$7.699\times n^{1.484}\,\ln\!\bigl(\tfrac{X_0}{\varepsilon}\bigr)+1$. This bound exhibits an $\mathcal{O}(n^{1.484})$ growth rate, which exceeds the $\mathcal{O}(n\log n)$ result known for the discrete OneMinMax within a reasonable margin. Such an increase in exponent is expected because continuous optimization features an infinite decision space, which requires finer‑grained exploration, and thus more generations to converge. Importantly, both the discrete and continuous bounds remain low‑order polynomials, and their asymptotic trends align closely, providing strong evidence for the validity of our estimation approach.

\subsubsection{Empirical Validation}
Tables \uppercase\expandafter{\romannumeral1} to \uppercase\expandafter{\romannumeral5} present comprehensive estimation results for five representative MOEAs (NSGA-II, MOEA/D, AR-MOEA, AGEMOEA-II, and PREA) across ten benchmark problems from the ZDT and DTLZ test suites. Each table displays the estimated running time expressions derived through surface fitting, and their corresponding numerical validation results. The numerical experimental data (standard deviation and mean from 100 independent runs) and the computed estimation results are obtained using specific problem dimensions: $n=10$ for ZDT4 and ZDT6, and $n=15$ for all other problems, selected based on their distinct decision variable domain characteristics. The mathematical expressions follow the form established in Eq. (\ref{eq15}), where $n$ represents problem dimension, $\varepsilon$ is the target precision, and $X_0$ is the initial population's IGD value. We employ a modified coefficient of determination $R^2$ to evaluate estimation quality \cite{carpenter1960principles, slinker1990primer}:
\begin{equation}
	\label{cofdet}
	R^2 = 1 - \frac{\sum_{i=1}^{\kappa} \left( \lg(f_s(n_i, \psi_i)) - \lg(g_i) \right)^2}{\sum_{i=1}^{\kappa} \left( \lg(g_i) - \frac{1}{\kappa} \sum_{i=1}^{\kappa} \lg(g_i) \right)^2}
\end{equation}
where $f_s$ is the fitted surface function. An estimation is considered correct when the empirical average FHT does not exceed the estimated upper bound and $R^2 > 0$, indicating reliable surface fitting.

Different algorithms use varying target precisions $\varepsilon$ based on their optimization capabilities. For instance, on DTLZ1, we set $\varepsilon = 0.04$ for PREA and $\varepsilon = 0.05$ for MOEA/D, as MOEA/D requires significantly more generations to reach higher precision levels. Despite the moderate precision requirements, the achieved solution quality demonstrates high overlap with the true Pareto front as shown in Fig. \ref{fig_si}. Our method employs surface fitting to approximate the gain distribution function based on statistical experimental data, and by substituting this fitted function for the unknown function $h(z)$ in Theorem \ref{theorem:1}, we derive the upper bound on EFHT. Consequently, the tightness of the estimated upper bound is closely related to surface fitting accuracy.

\begin{figure}[!b]
	\centering
	\subfloat[]{\includegraphics[width=1.6in]{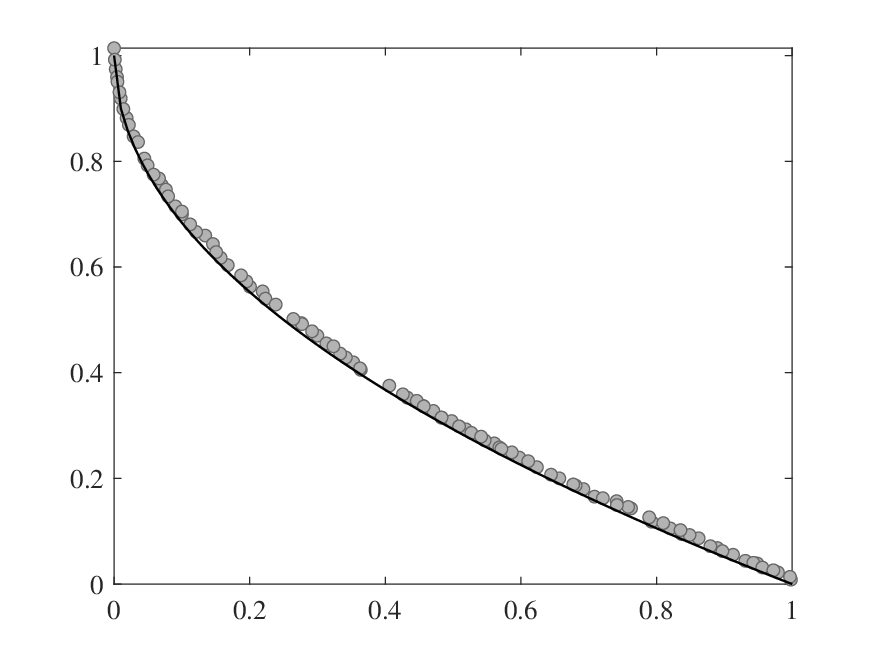}%
		\label{fig_first_c}}
	\hfil
	\subfloat[]{\includegraphics[width=1.6in]{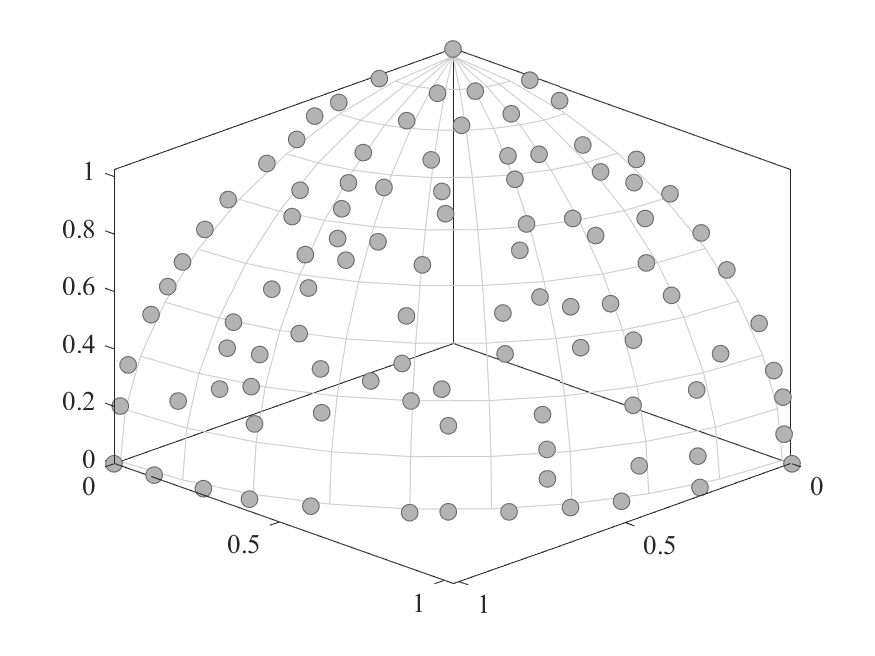}%
		\label{fig_second_c}}
	\caption{Optimization results. (a) NSGA-\uppercase\expandafter{\romannumeral2} on ZDT1. (b) PREA on DTLZ2.}
	\label{fig_si}
\end{figure}

From Tables \uppercase\expandafter{\romannumeral1} to \uppercase\expandafter{\romannumeral5}, we observe that all estimations with positive $R^2$ values are correct. For cases where $R^2 > 0$, the estimated values remain within reasonable magnitudes compared to numerical results, typically not exceeding them by more than one order of magnitude. Positive $R^2$ values indicate accurate surface fitting and lead to tighter estimated upper bounds on EFHT.
Conversely, negative $R^2$ values suggest significant discrepancies between fitted surfaces and actual data, resulting in excessively loose upper bounds. Such estimations significantly exceed numerical results and are excluded from practical consideration due to limited utility \cite{orelien2008fixed}.

Based on the $R^2$ criterion, we can effectively filter out unreliable estimates while retaining meaningful bounds. The results demonstrate that our approach successfully estimates upper bounds on EFHT across all three main MOEA categories: dominance-based, decomposition-based, and indicator-based methods. This broad validation confirms the generality of our proposed framework.

\begin{table*}[b]
    \centering
    \caption{Estimation Comparison of NSGA-II and PREA on Test Problems}
    \resizebox{0.9\textwidth}{!}{
    \begin{tabular}{|l|c|c|c|c|c|}
        \hline
        \multirow{2}{*}{\textbf{Fitness Function}} & \multicolumn{2}{c|}{\textbf{NSGA-II}} & \multicolumn{2}{c|}{\textbf{PREA}} & \multirow{2}{*}{\textbf{Comparison results (Winner)}} \\ \cline{2-5}
        & \textbf{Running Time} & \textbf{Value} & \textbf{Running Time} & \textbf{Value} & \\ \hline
        ZDT1 & $1.423 \times n^{1.234} \ln \left(\frac{X_0}{\epsilon}\right) + 1$  & 2.00E+4 & $4.734 \times n^{1.046} \ln \left(\frac{X_{0}}{\varepsilon}\right)+1$ & 3.88E+04 & NSGA-II \\ \hline
        ZDT2 & $4.637 \times n^{1.015} \ln \left(\frac{X_0}{\epsilon}\right) + 1$  & 4.00E+4 & $5.701 \times n^{1.103} \ln \left(\frac{X_{0}}{\varepsilon}\right)+1$ & 6.03E+04 & NSGA-II \\ \hline
        ZDT3 & $5.477 \times n^{0.929} \ln \left(\frac{X_0}{\epsilon}\right) + 1$ & 3.09E+04 & $5.477 \times n^{1.154} \ln \left(\frac{X_{0}}{\varepsilon}\right)+1$ & 5.71E+04 & NSGA-II \\ \hline
        ZDT4 & $30.000 \times n^{0.583} \ln \left(\frac{X_{0}}{\varepsilon}\right)+1$  & 9.95E+04 & $30.000 \times n^{0.719} \ln \left(\frac{X_{0}}{\varepsilon}\right)+1$ & 1.33E+05 & NSGA-II \\ \hline
        ZDT6 & $15.290 \times n^{1.045} \ln \left(\frac{X_{0}}{\varepsilon}\right)+1$  & 1.09E+05 & $10.166 \times n^{1.328} \ln \left(\frac{X_{0}}{\varepsilon}\right)+1$ & 1.38E+05 & NSGA-II \\ \hline
    \end{tabular}
    }
    \label{tab:comparison1}
\end{table*}

\begin{table*}[b]
    \centering
    \caption{Numerical Comparison of NSGA-II and PREA on Test Problems Conducted in This Study}
    \resizebox{0.9\textwidth}{!}{
    \begin{tabular}{|l|c|c|c|c|c|c|}
        \hline
        \multirow{2}{*}{\textbf{Test Problem}} & \multicolumn{2}{c|}{\textbf{NSGA-II}} & \multicolumn{2}{c|}{\textbf{PREA}} & \multirow{2}{*}{\textbf{Comparison results (Winner)}} & \multirow{2}{*}{\textbf{Consistency}} \\ \cline{2-5}
        & \textbf{Mean} & \textbf{St.D} & \textbf{Mean} & \textbf{St.D} & & \\ \hline
        ZDT1 & 5.78E+03 & 5.07E+02 & 1.35E+04 & 4.0765E+03 & NSGA-II & consistent \\ \hline
        ZDT2 & 6.70E+03 & 1.15E+03 & 1.73E+04 & 2.85E+03 & NSGA-II & consistent \\ \hline
        ZDT3 & 7.36E+03 & 6.57E+03 & 2.33E+04 & 1.87E+04 & NSGA-II & consistent \\ \hline
        ZDT4 & 1.84E+04 & 3.12E+03 & 3.20E+04 & 5.26E+03 & NSGA-II & consistent \\ \hline
        ZDT6 & 1.44E+04 & 1.49E+03 & 1.92E+04 & 2.02E+03 & NSGA-II & consistent \\ \hline
    \end{tabular}
    }
    \label{tab:comparison2}
\end{table*}

\subsection{Running Time Comparison of NSGA-II and PREA}
The proposed method enables practical performance comparison between different algorithms, providing valuable guidance for algorithm selection in engineering applications. Due to the varying performance of each algorithm, the achievable target precision also differs. To obtain more reliable estimation results, we set different target values for different algorithms in both the estimation experiments and the corresponding numerical experiments. This aims to maximize target precision while balancing computational time. Therefore, the previous experimental results cannot be directly used for running time comparisons. 

We conduct a comparative study using ZDT test problems to evaluate NSGA-II and PREA under identical conditions: population size of 100, problem dimension of 15, and target precision $\varepsilon = 0.01$. Each numerical experiment is repeated 100 times to ensure statistical robustness. The estimation accuracy is validated by verifying that the algorithm ranking from our estimation method matches the ranking obtained from numerical experiments.

Tables \ref{tab:comparison1} and \ref{tab:comparison2} demonstrate consistent results across all test problems, with NSGA-II consistently outperforming PREA in both estimated and empirical running times. This perfect consistency between estimation and numerical results validates the effectiveness of our proposed approach for runtime assessment and algorithm comparison in multi-objective optimization.

\subsection{Stability of the Estimation}

We demonstrate the stability of our estimation approach through repeated experiments. While sampling-based statistical methods inherently involve randomness, our adaptive sampling method effectively mitigates this variability, ensuring reliable estimation results.

To assess stability, we performed 30 repeated EFHT estimations for PREA on ZDT1 with consistent parameter settings. The coefficient of variation (CV), defined as the ratio of standard deviation to mean \cite{pearson1896vii}, serves as our stability metric:
\begin{equation}\label{eq16}
	CV = \frac{\sigma_0}{\mu}
\end{equation}
where $\sigma_0$ is the standard deviation and $\mu$ is the mean.

Figure \ref{fig_sim} illustrates four representative surface fitting results, with corresponding estimation expressions shown in Table \ref{tab1}. All expressions maintain the same structural form, with variations only in the coefficient and the exponent of $n$. Across 30 trials, we calculated CV values of 0.152 for the coefficient and 0.181 for the exponent, indicating acceptable variability levels. These results confirm the stability and reliability of our adaptive sampling method, validating its effectiveness in ensuring consistent estimation performance.

\begin{table}[h]
	\begin{center}
		\caption{Comparison of Four Replicated Experiments}
		\label{tab1}
		\begin{tabular}{| c | c |}
			\hline
			Experiment No. & Running Time \\ \hline
			1 & $4.496 \times n^{1.001} \ln \left(\frac{X_{0}}{\varepsilon}\right)+1$ \\ \hline
			2 & $5.434 \times n^{0.992} \ln \left(\frac{X_{0}}{\varepsilon}\right)+1$ \\ \hline
			3 & $5.477 \times n^{0.899} \ln \left(\frac{X_{0}}{\varepsilon}\right)+1$ \\ \hline
			4 & $5.477 \times n^{0.905} \ln \left(\frac{X_{0}}{\varepsilon}\right)+1$ \\ \hline
		\end{tabular}
	\end{center}
\end{table}

\begin{figure}[h]
	\centering
	\subfloat[]{\includegraphics[width=1.7in]{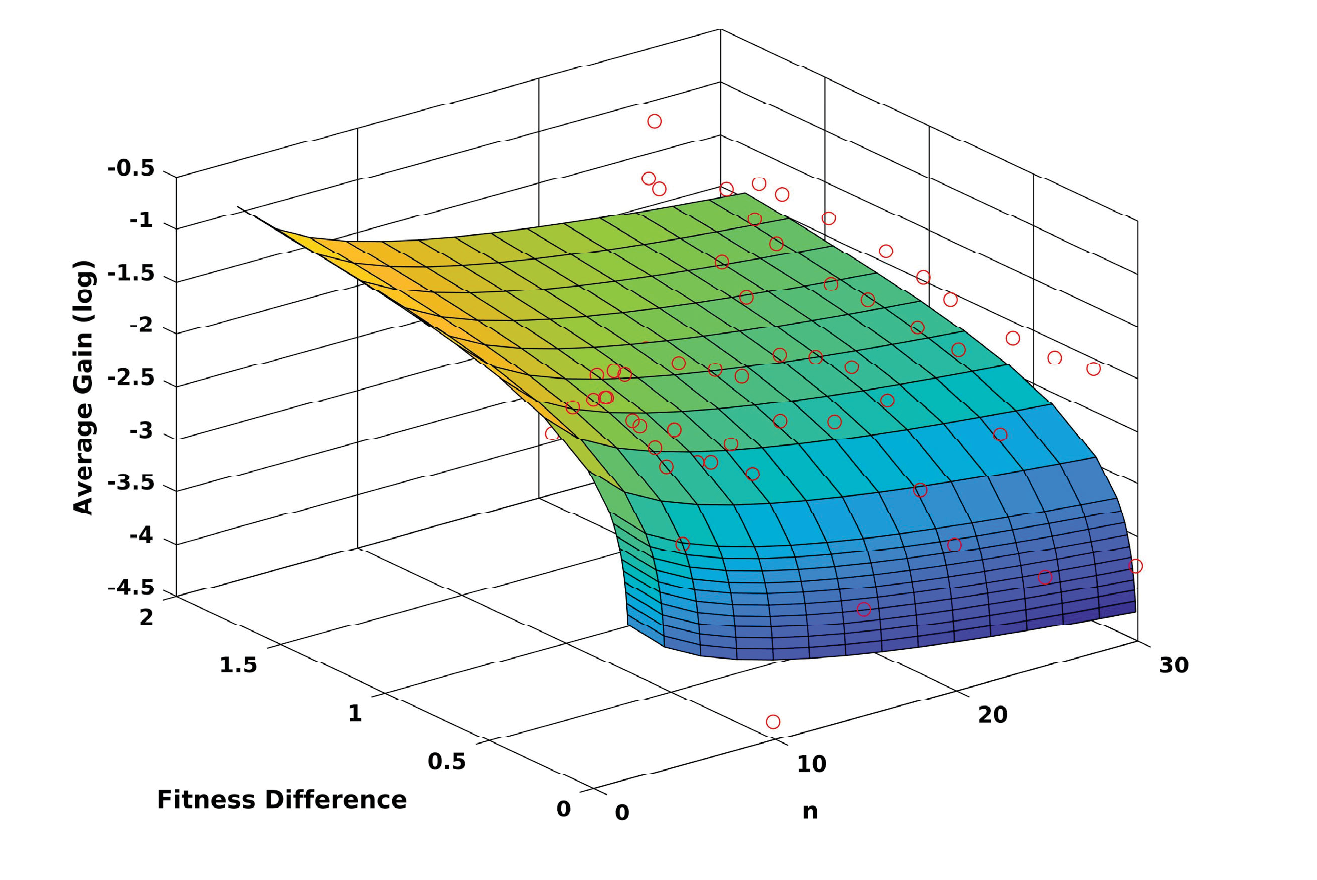}%
		\label{fig_first_case}}
	\hfil
	\subfloat[]{\includegraphics[width=1.7in]{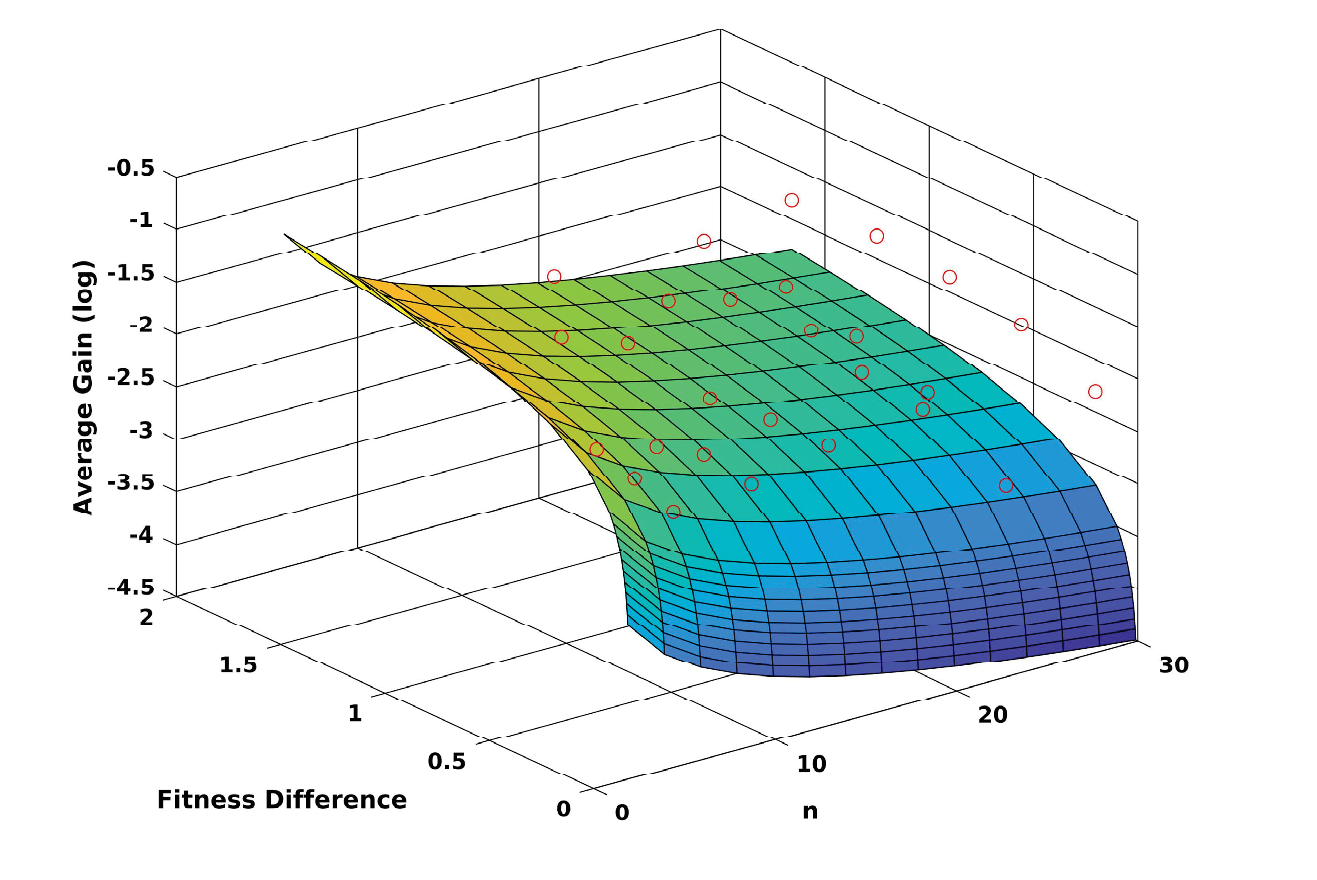}%
		\label{fig_second_case}}
	\hfil
	\subfloat[]{\includegraphics[width=1.7in]{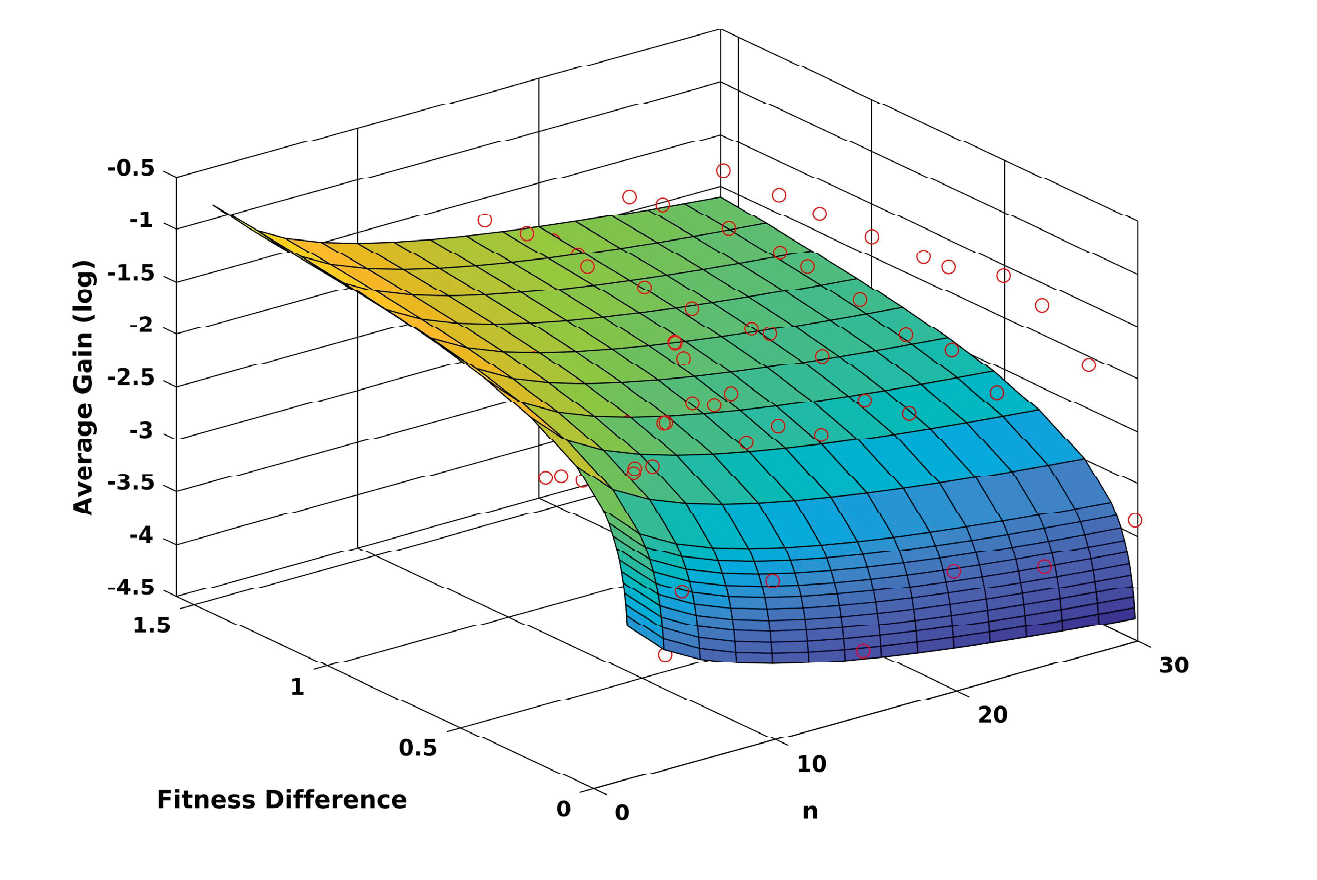}%
		\label{fig_third_case}}
	\hfil
	\subfloat[]{\includegraphics[width=1.7in]{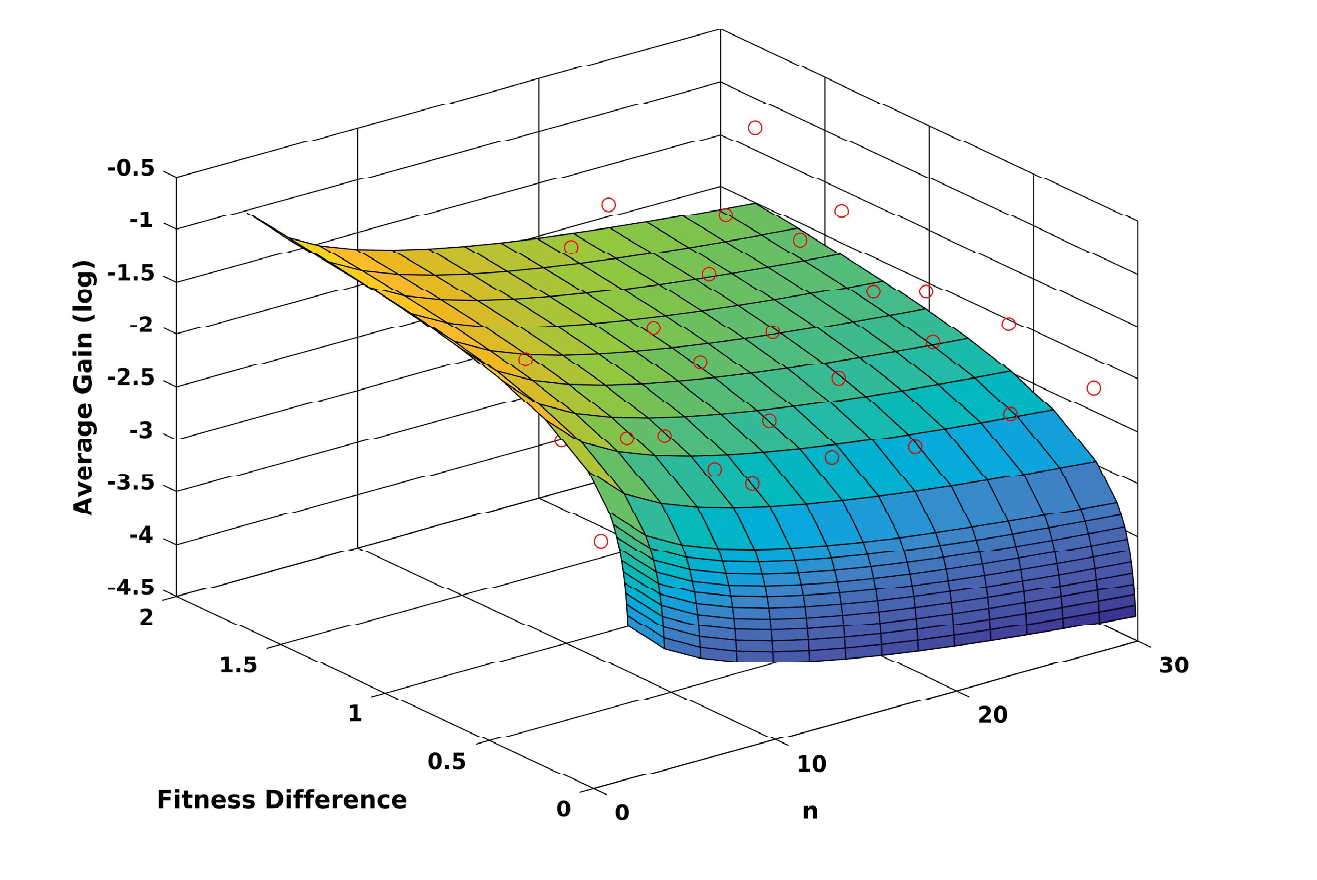}%
		\label{fig_fourth_case}}
	\caption{Surface fitting results of four replicated experiments. (a) No.1. (b) No.2. (c) No.3. (d) No.4.}
	\label{fig_sim}
\end{figure}

\section{Conclusion}\label{conclusion}
Most existing running-time analyses of MOEAs suffer from two primary limitations: dependence on algorithm-specific or problem-specific simplifications, and focus predominantly on combinatorial optimization problems. To address these limitations, we propose a general experimental framework for estimating upper bounds on the EFHT of MOEAs in numerical optimization without requiring simplifying assumptions about algorithms or problems. Our approach establishes an average gain model for MOEAs using the IGD metric as a progress measurement framework. Through statistical experiments, we empirically estimate the distribution functions of relevant random variables and derive upper bounds on EFHT based on drift analysis principles. An adaptive sampling method is introduced to enhance the stability and accuracy of the statistical estimation process.

Comprehensive experiments on ZDT and DTLZ benchmark suites with the five representative MOEAs validate the effectiveness of our approach. The results demonstrate successful upper bound estimation across different algorithmic paradigms including dominance-based, decomposition-based, and indicator-based methods, confirming the generality of our framework. Stability analysis through repeated experiments further validates the reliability of our adaptive sampling methodology.

The proposed framework provides a practical tool for runtime estimation in continuous multi-objective optimization, offering valuable guidance for algorithm selection and resource allocation in engineering applications. As a complement to theoretical analysis, it bridges the gap between rigorous mathematical approaches and practical implementation needs. Future work will extend the average gain theory to multi-task and multi-modal optimization contexts, where theoretical runtime analysis remains underdeveloped.
\bibliographystyle{IEEEtran}
\bibliography{references.bib}

\end{document}